%% file: main.tex
\documentclass[sigconf]{acmart}

\usepackage{booktabs} 

\setcopyright{rightsretained}

\usepackage[utf8]{inputenc}
\usepackage[T1]{fontenc}

\newtheorem{theorem}{Theorem}[section]
\newtheorem{proposition}[theorem]{Proposition}

\usepackage[colorinlistoftodos]{todonotes}

\acmDOI{ }

\acmISBN{ }

\acmConference[Preprint]{Preprint}
\acmYear{prelimiary work}
\copyrightyear{prelimiary work}

\acmPrice{ }

\begin{document}
\title{Differential Evolution with Reversible Linear Transformations}

\author{Jakub M. Tomczak}
\affiliation{%
  \institution{Vrije Universiteit Amsterdam}
  \city{Amsterdam} 
  \country{the Netherlands}
}
\email{j.m.tomczak@vu.nl}

\author{Ewelina W\k{e}glarz-Tomczak}
\affiliation{%
  \institution{University of Amsterdam}
  \city{Amsterdam} 
  \country{the Netherlands}
}
\email{e.m.weglarz-tomczak@uva.nl}

\author{Agoston E. Eiben}
\affiliation{%
  \institution{Vrije Universiteit Amsterdam}
  \city{Amsterdam} 
  \country{the Netherlands}
}
\email{a.e.eiben@vu.nl}

\renewcommand{\shortauthors}{Tomczak et al.}

\begin{abstract}
\input{abstract.tex}
\end{abstract}

%
%



\maketitle

\section{Introduction}
\input{introduction.tex}

\section{Background}
\input{background.tex}

\section{Our Approach}
\input{our_approach.tex}

\section{Experiments}
\input{experiments.tex}

\section{Conclusion}
\input{conclusion.tex}

\section*{Appendix}
\input{appendix.tex}
\begin{acks}
EW-T is financed by a grant within Mobilnosc Plus V from the Polish Ministry of Science and Higher Education (Grant No. 1639/MOB/V/2017/0).
\end{acks}

\bibliographystyle{ACM-Reference-Format}
\bibliography{main.bib}

\end{document}

%% file: abstract.tex
Differential evolution (DE) is a well-known type of evolutionary algorithms (EA). Similarly to other EA variants it can suffer from small populations and loose diversity too quickly. This paper presents a new approach to mitigate this issue: We propose to generate new candidate solutions by utilizing reversible linear transformation applied to a triplet of solutions from the population. In other words, the population is enlarged by using newly generated individuals without evaluating their fitness. We assess our methods on three problems: (i) benchmark function optimization, (ii) discovering parameter values of the gene repressilator system, (iii) learning neural networks. The empirical results indicate that the proposed approach outperforms vanilla DE and a version of DE with applying differential mutation three times on all testbeds.

%% file: introduction.tex

Optimization is about finding a solution that minimizes (or maximizes) an objective function for given constraints, \textit{i.e.}, possible values that solutions can take. A subset of optimization problems with so called \textit{black-box} objective functions constitute \textit{black-box} optimization. In general, a black-box is any process that when provided an input, returns an output, but its analytical description is unavailable or it is non-differentiable \cite{audet2017derivative}. Examples of black-box functions (and/or constraints) are computer programs \cite{cranmer2019frontier}, physical and biochemical processes \cite{tomczak2019estimating, toni2009approximate}, or evolutionary robotics \cite{doncieux2015evolutionary}. 

There exists a vast of derivative-free optimization (DFO) methods, ranging from classical algorithms like iterative local search or direct search \cite{audet2017derivative} to modern approaches like Bayesian optimization \cite{shahriari2015taking} and evolutionary algorithms (EA) \cite{back1996evolutionary, eiben2003introduction}. Differential evolution (DE) \cite{storn1997differential, price2006differential} is one of the most successful and popular population-based DFO algorithms that utilizes evolutionary operators (mutation and crossover) and a selection mechanism to generate a new set of candidate solutions. DE is a metaheuristic with no convergence guarantees, however, it possesses multiple interesting theoretical properties \cite{opara2019differential}. Since the original publication of DE \cite{storn1997differential}, the method was extended in many ways by, \textit{e.g.}, using adaptive local search \cite{noman2008accelerating}, modifying the differential mutation operator \cite{zhang2009jade} or optimizing parameters of DE \cite{sharma2019deep}. DE has been also applied to many real-life problem, such as, digital filter design \cite{karaboga2005digital}, parameter estimation in ODEs \cite{wang2001kinetic}, discovering predictive genes in microarray data \cite{tasoulis2006differential}, or robot navigation \cite{martinez2018robot}.

In this paper, we follow this line of research and present an extension of DE. One of potential issues with DE is, similarly to other EA variants, that it can suffer from small populations and loose diversity too quickly. Therefore, a potential solution is adjusting or modifying the population size \cite{eiben2004evolutionary}. Here, we propose to enlarge the population \textit{on-the-fly} by generating new candidate solutions using reversible linear transformation applied to a triplet of solutions from the population. As a result, we take a population of $N$ individuals and generate $3N$ new candidates assuming that we can afford running extra evaluations. This procedure allows to enhance DE and explore/exploit the search space better. We evaluate our approach on three testbeds. First, we present results on benchmark function optimization (Griewank, Rastrigin, Schwefel and Salomon functions). Second, we apply the proposed methods to discovering parameter values of the gene repressilator system. Lastly, we utilize the new DE schema in learning neural networks on image data. In all experiments, we show that enlarging the population size indeed allow to faster convergence (in terms of the number of fitness evaluations) and the reversible linear transformations provide efficient and effective alternative to the vanilla differential mutation.

The contribution of the paper is threefold. First, we propose to enhance DE by applying reversible linear transformations with two different linear operators. Second, we analyze the operators by inspecting their eigenvalues. Third, we show empirically on problems with the number of variables ranging from $4$ to $4120$ that the proposed DE with reversible linear transformations significantly outperforms the DE and its extension with three perturbations.

%% file: background.tex
\subsection{Black-box optimization}
We consider an optimization problem of a function $f: \mathbb{X} \rightarrow \mathbb{R}$, where $\mathbb{X} \subseteq \mathbb{R}^{D}$ is the search space. In this paper we focus on the minimization problem, namely:
\begin{equation}
    \mathbf{x}^{*} = \arg\min_{\mathbf{x} \in \mathbb{X}} f(\mathbf{x}).
\end{equation}
Further, we assume that the analytical form of the function $f$ is unknown or cannot be used to calculate derivatives, however, we can query it through a simulation or experimental measurements. Problems of this sort are known as \textit{black-box optimization problems} \cite{audet2017derivative, jones1998efficient}. Additionally, we consider a bounded search space, \textit{i.e.}, we include inequality constraints for all dimensions in the following form: $l_d \leq x_d \leq u_d$, where $l_d, u_d \in \mathbb{R}$ and $l_d < u_d$, for $d=1, 2, \ldots, D$ .

\subsection{Differential Evolution} 
One of the most widely-used methods for black-box optimization problems is \textit{differential evolution} (DE) \cite{storn1997differential} that requires a population of candidate solutions, $\mathcal{X} = \{\mathbf{x}_{1}, \ldots, \mathbf{x}_{N}\}$, to iteratively generate new query points. A new candidate is generated by randomly picking a triple from the population, $(\mathbf{x}_i, \mathbf{x}_j, \mathbf{x}_k) \in \mathcal{X}$, and then $\mathbf{x}_i$ is perturbed by adding a scaled difference between $\mathbf{x}_j$ and $\mathbf{x}_k$, that is:
\begin{equation}\label{eq:mutation}
    \mathbf{y} = \mathbf{x}_i + F ( \mathbf{x}_j - \mathbf{x}_k ) ,
\end{equation}
where $F \in \mathbb{R}_{+}$ is the scaling factor. This operation could be seen as an adaptive \textit{mutation operator} that is widely known as \textit{differential mutation} \cite{price2006differential}.

Further, the authors of \cite{storn1997differential} proposed to sample a binary mask $\mathbf{m} \in \{0, 1\}^{D}$ according to the Bernoulli distribution with probability $p = P(m_d = 1)$ shared across all $D$ dimensions, and calculate the final candidate according to the following formula:
\begin{equation}\label{eq:crossover}
    \mathbf{v} = \mathbf{m} \odot \mathbf{y} + (1 - \mathbf{m}) \odot \mathbf{x}_{i},
\end{equation}
where $\odot$ denotes the element-wise multiplication. In the evolutionary computation literature this operation is known as \textit{uniform crossover operator} \cite{back1996evolutionary, eiben2003introduction}. In this paper, we fix $p=0.9$ following general recommendations in literature \cite{pedersen2010good} and use the uniform crossover in all methods.

The last component of a population-based method is a selection mechanism. There are multiple variants of selection \cite{back1996evolutionary, eiben2003introduction},  however, here we use the ``survival of the fittest'' approach, \textit{i.e.}, we combine the old population with the new one and select $N$ candidates with highest fitness values (\textit{i.e.}, the deterministic $(\mu + \lambda)$ selection).

This variant of DE is referred to as “DE/\textit{rand}/\textit{1}/\textit{bin}”, where \textit{rand} stands for randomly selecting a base vector, \textit{1} is for adding a single perturbation (a vector difference) and \textit{bin} denotes the uniform crossover. Sometimes it is called \textit{classic DE} \cite{price2006differential}.

%% file: our_approach.tex
\begin{figure*}[!htp]
    \centering
    \includegraphics[width=1\textwidth]{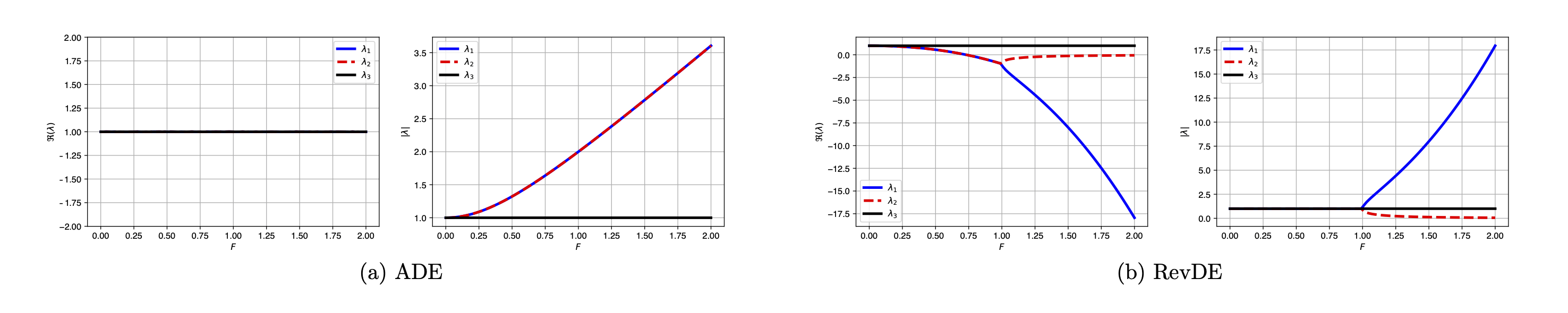}
    \vskip -4mm
    \caption{Real part of eigenvalues, $\Re(\lambda)$, and sbsolute value of eigenvalues, $|\lambda|$, for: (a) $\mathbf{M}$ in ADE, and (b) $\mathbf{R}$ in RevDE.}
    \label{fig:eigenvalues}
\end{figure*}

Generating new candidates in DE requires sampling a triplet of solutions and, basing on these points, one solution is perturbed using the other two solutions. This approach possesses multiple advantages, naming only a few:
\begin{itemize}
    \item[(i)] it is non-parametric, \textit{i.e.}, contrary to evolutionary strategies \cite{back2013contemporary}, no assumption on the underlying distribution of the population is made;
    \item[(ii)] it has been shown to be effective in many benchmark optimization problems and real-life applications \cite{price2006differential}.
\end{itemize}
However, the number of possible perturbations is finite and relies entirely on the population size. Therefore, a small population size could produce insufficient variability of new candidate solutions. To counteract this issue, we propose the following solutions:
\begin{enumerate}
    \item In order to increase variability, we can perturb candidates multiple times by running the differential mutation more than once (\textit{e.g.}, three times).
    \item In fact, we can use the selected triple of points and use it three times to generate new points. In other words, we notice that there is no need to sample three different triplets.
    \item We propose to modify the selected triplet by using generated new solutions \textit{on-the-fly}. This approach allows to enlarge the population size.
\end{enumerate}
In the following subsections, we outline the three approaches. Further, we notice that the second and the third method could be represented as linear transformations. As such, we could analyze them algebraically. 

\subsection{Differential Evolution x3}
In the first approach we generate a larger new population by perturbing the point $\mathbf{x}_i$ using multiple candidate solutions, namely, $\mathbf{x}_j, \mathbf{x}_k, \mathbf{x}_l, \mathbf{x}_m, \mathbf{x}_n, \mathbf{x}_q \in \mathcal{X}$. Then, we can produce $3N$ new candidate solutions instead of $N$ as follows:
\begin{align}\label{eq:dex3}
    \mathbf{y}_1 &= \mathbf{x}_i + F ( \mathbf{x}_j - \mathbf{x}_k ) \\
    \mathbf{y}_2 &= \mathbf{x}_i + F ( \mathbf{x}_l - \mathbf{x}_m ) \\
    \mathbf{y}_3 &= \mathbf{x}_i + F ( \mathbf{x}_n - \mathbf{x}_q ) .
\end{align}
This approach requires sampling more pairs and evaluating more points, however, it allows to better explore the search space. We refer to this approach as \textit{Differential Evolution} $\times 3$, or DEx3 for short.

\subsection{Antisymmetric Differential Evolution}
We first notice that in the DEx3 approach we sample three pairs of points to calculate perturbations. Since we pick them at random, we propose to sample three candidates $\mathbf{x}_i, \mathbf{x}_j, \mathbf{x}_k \in \mathcal{X}$ and calculate perturbations by changing their positions only, that is:
\begin{align}\label{eq:raw_ade}
    \mathbf{y}_1 &= \mathbf{x}_i + F(\mathbf{x}_j - \mathbf{x}_k) \nonumber \\
    \mathbf{y}_2 &= \mathbf{x}_j + F(\mathbf{x}_k - \mathbf{x}_i) \\
    \mathbf{y}_3 &= \mathbf{x}_k + F(\mathbf{x}_i - \mathbf{x}_j) . \nonumber
\end{align}
In other words, we perturb each point by using the remaining two.
Interestingly, we notice that Eq. \ref{eq:raw_ade} corresponds to applying a linear transformation to these three points. For this purpose, we rewrite (\ref{eq:raw_ade}) using matrix notation by introducing matrices $\mathbf{Y} = [\mathbf{y}_1, \mathbf{y}_2, \mathbf{y}_3]^{\top}$ and $\mathbf{X} = [\mathbf{x}_i, \mathbf{x}_j, \mathbf{x}_k]^{\top}$ that yields:
\begin{equation}\label{eq:linear_transformation}
    \mathbf{Y}
     = \mathbf{M}
        \mathbf{X} ,
\end{equation}
where:
\begin{equation}\label{eq:eye_anti}
    \mathbf{M} = \begin{bmatrix}
        1  & F  & -F \\
        -F & 1  & F \\
        F  & -F & 1
        \end{bmatrix} .
\end{equation}
The matrix $\mathbf{M}$ can be further decomposed as follows:
\begin{equation}\label{eq:m_decomposed}
    \mathbf{M} = \underbrace{\begin{bmatrix}
        1  & 0  & 0 \\
        0 & 1  & 0 \\
        0  & 0 & 1
        \end{bmatrix}}_{\mathbf{I}}
        +
        \underbrace{\begin{bmatrix}
        0  & F  & -F \\
        -F & 0  & F \\
        F  & -F & 0
        \end{bmatrix}}_{\mathbf{A}} ,
\end{equation}
where $\mathbf{I}$ denotes the identity matrix, and $\mathbf{A}$ is the antisymmetric matrix.

Comparing Eq. \ref{eq:raw_ade} to DE$\times 3$ we notice that there is no need to sample additional candidates beyond one triplet. Moreover, the new mutation in (\ref{eq:raw_ade}) allows us to analyze the transformation from the algebraic perspective. Additional interesting property following from representing DE using a linear operator corresponds to parallelization of calculations and, thus, it could greatly speed up computations.

We refer to this version of DE as \textit{Antisymmetric Differential Evolution} (ADE), because the linear transformation consists of the identity matrix and the antisymmetric matrix parameterized with the scaling factor $F$.

\subsection{Reversible Differential Evolution}
The linear transformation presented in Eq. \ref{eq:raw_ade} allows to utilize the triplet $(\mathbf{x}_i, \mathbf{x}_j, \mathbf{x}_k)$ to generate three new points, however, it could be still seen as applying DE three times, but in a specific manner (\textit{i.e.}, by defining the linear operator $\mathbf{M}$). A natural question arises whether a different transformation could be proposed that allows \textit{better} exploitation and/or exploration of the search space. The mutation operator in DE perturbs candidates using other individuals in the population. As a result, having too small population could limit exploration of the search space. In order to overcome this issue, we propose to modify ADE by using newly generated candidates \textit{on-the-fly}, that is:
\begin{align}\label{eq:revde_raw}
    \mathbf{y}_1 &= x_i + F(\mathbf{x}_j - \mathbf{x}_k) \nonumber \\
    \mathbf{y}_2 &= x_j + F(\mathbf{x}_k - \mathbf{y}_1) \\
    \mathbf{y}_3 &= x_k + F(\mathbf{y}_1 - \mathbf{y}_2) .\nonumber
\end{align}
Using new candidates $\mathbf{y}_1$ and $\mathbf{y}_2$ allows to calculate perturbations using points outside the population. This approach does not follow a typical construction of an EA where only evaluated candidates are mutated. Further, similarly to ADE, we can express (\ref{eq:revde_raw}) as a linear transformation $\mathbf{Y} = \mathbf{R} \mathbf{X}$ with the following linear operator:
\begin{equation}\label{eq:r}
    \mathbf{R}
     = 
    \begin{bmatrix}
        1       & F              & -F \\
        -F      & 1 - F^2        & F + F^2 \\
        F + F^2 & -F + F^2 + F^3 & 1 - 2 F^2 - F^3
        \end{bmatrix} .
\end{equation}
In order to obtain the matrix $\mathbf{R}$, we need to plug $\mathbf{y}_1$ to the second and third equation in (\ref{eq:revde_raw}), and then $\mathbf{y}_2$ to the last equation in (\ref{eq:revde_raw}).

We refer to this version of DE as \textit{Reversible Differential Evolution} (RevDE), because the linear transformation is reversible (see next subsection).

\subsection{Algebraic properties of ADE and RevDE}
\label{sec:algebraic_properties}

\subsubsection{Reversibility} 
An interesting property of the matrices $\mathbf{M}$ and $\mathbf{R}$ in ADE and RevDE, respectively, is that they are nonsigular matrices (see the Appendix for the proofs). Since they are non-singular, they are also invertible, and, thus, ADE and RevDE use reversible linear transformations.

The reversibility is an important property for formulating Markov Chain Monte Carlo methods (MCMC) \cite{andrieu2003introduction}. Therefore, we could take advantage of the proposed reversible linear transformations and extend the existing work on utilizing DE for sampling methods \cite{ter2006markov}. However, this is beyond the scope of this paper and we leave investigating it in the future work.

\subsubsection{Analysis of eigenvalues}
We can obtain an insight into a linear operator by analysing its eigenvalues that tell us how a matrix transforms an object \cite{strang1993introduction}. Therefore, they play a crucial role in analyzing properties of linear operators, \textit{e.g.}, in control theory real parts of eigenvalues are used to determine stability of linear dynamical systems (if real part of all eigenvalues are lower than $0$, then the system is stable; otherwise it is unstable \cite{bubnicki2005modern}). Further, the absolute value of an eigenvalue $\lambda_i$ determines the influence of the corresponding eigenvector \cite{strang1993introduction}. If the absolute value of the eigenvector is lower than $1$, then the eigenvector is a decaying mode. Similarly, if $|\lambda_{i}| > 1$, then the eigenvector is a dominant mode. In the case of $|\lambda_i| = 1$ we call its eigenvector a steady state.

In Figure \ref{fig:eigenvalues} we present the absolute value and the real part of eigenvalues for $\mathbf{M}$ and $\mathbf{R}$. We notice the following facts:
\begin{itemize}
    \item For ADE, all real parts of eigenvalues are above or equal $1$, and all absolute values of eigenvalues are equal $1$. As a result, the method will never lead to a decaying mode, and as such it will encourage exploration of the search space.
    \item For RevADE, the situation is different, namely, for $F<0.75$ all real parts and all absolute values of eigenvalues are positive, while for $F>0.75$ one eigenvalue has a real part equal $1$ and the other two eigenvalues have real parts lower than $0$. However, in all cases, all absolute values of eigenvalues are larger than 0.\footnote{This fact follows from the non-singularity of the matrix $\mathbf{R}$, \textit{i.e.}, a matrix is non-singular iff all its eigenvalues are non-zero.} In other words, RevDE for some values of $F$ possesses steady states, but for $F>1$ one eigenvalue blows up and leads to the dominant mode, while the other eigenvalue decays to $0$ resulting in a decaying mode.
\end{itemize}

This analysis suggests that, in the case of ADE, taking too large $F$ could result in generating candidate solutions that are dominated by a direction indicated by one of two eigenvectors. Consequently, this could lead to ``jumping'' in the search space. Since ADE is closely related to DE$\times 3$, this result sheds an additional light on the behavior of DE, and seems to confirm that taking $F$ larger than $0.5$ in DE is not a reasonable decision.

In the case of RevDE it seems that taking values of $F$ below $0.75$ is preferable, because then the linear operator will not lead to dominating modes. As a result, a better exploitation/exploration of the search space could be achieved. 



%% file: experiments.tex
In order to verify our approach empirically, we compare the three proposed methods and the standard DE on three testbeds:
\begin{enumerate}
    \item \textit{Benchmark functions}: selected benchmark function for optimization.
    \item \textit{Gene Repressilator System}: discovering parameter values of a system of ordinary differential equations for given observations.
    \item \textit{Neural Networks Learning}: learning a neural network with one hidden layer on image dataset.
\end{enumerate}

In all experiments, we used the uniform crossover with $p=0.9$ for all methods. The scale parameter $F$ was selected from the following range of values: $\{0.125, 0.25, 0.375, 0.5, 0.6, 0.625, 0.675, 0.75\}$. The population size was set to $500$ across all experiments. We pick the base vector \textit{randomly}.

The code of the methods and all experiments is available under the following link: \url{https://github.com/jmtomczak/reversible-de}.

\begin{figure*}[!htp]
    \centering
    \includegraphics[width=1\textwidth]{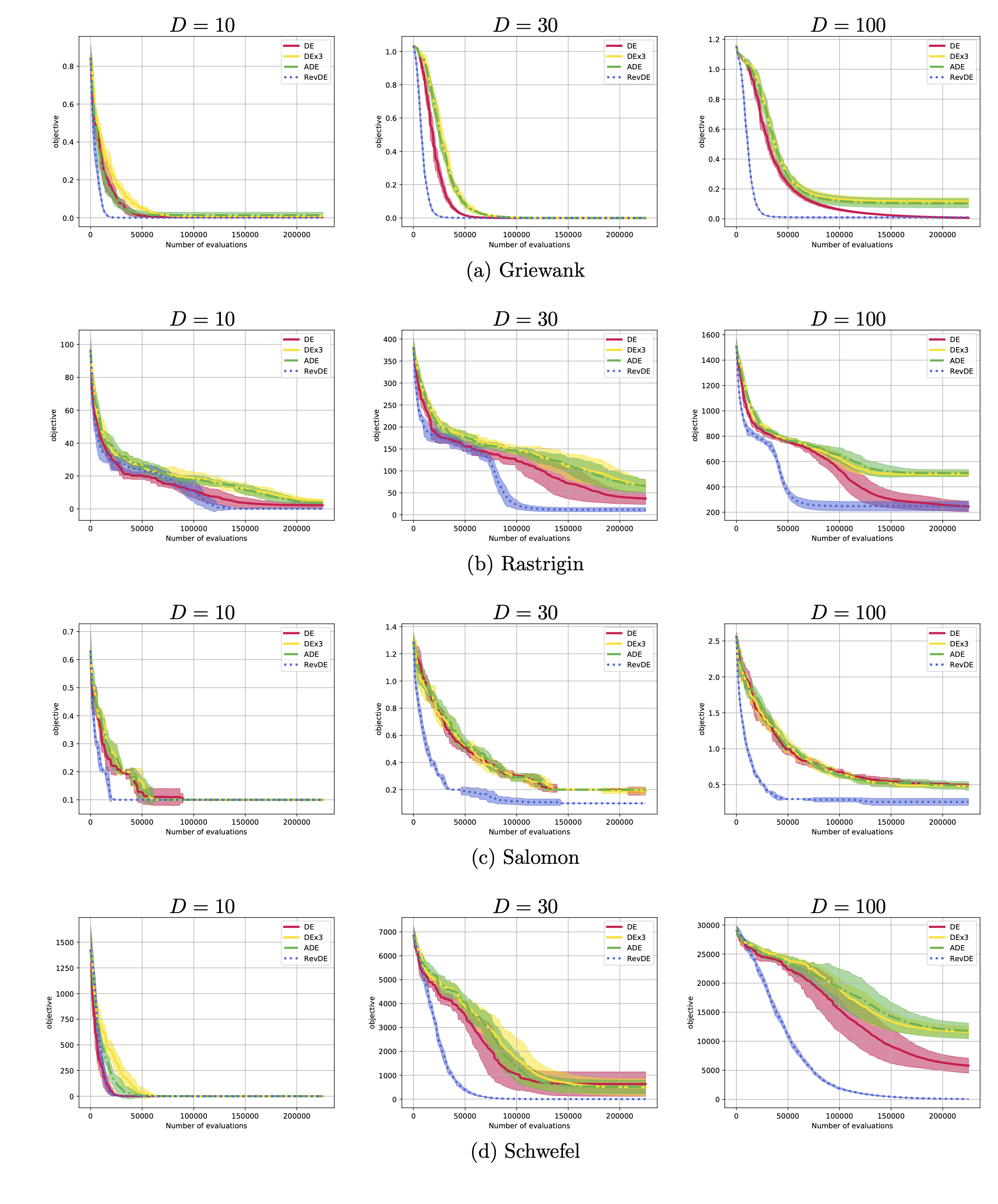}
    \vskip -6mm
    \caption{The results of the best solution found until given evaluation on four benchmark black-box optimization testbeds (a) Griewank function, (b) Rastrigin function, (c) Salomon function, (d) Schwefel function, and three cases: (left column) $10$D case, (middle column) $30$D case, (right column) $100$D case. The solid red lines correspond to DE, the solid yellow lines are for DEx3, the dotted-dashed green lines depict ADE, and the dotted blue lines represent RevDE. In all $12$ test cases an average and one standard deviations over $10$ runs are presented.}
    \label{fig:benchmark_results}
\end{figure*}

\subsection{Benchmark Function Optimization}

\subsubsection{Details} We evaluate the proposed methods on the optimization task of four benchmark functions:
\begin{itemize}
    \item Griewank function \cite{griewank1981generalized}:
    \begin{equation}
        f(\mathbf{x}) = 1 + \sum_{d=1}^{D} \sqrt{x_{d}^2 / 4000} - \prod_{d=1}^{D} \cos \Big{(} \frac{x_{d}}{\sqrt{d}} \Big{)}
    \end{equation}
    with the box constraints: $\forall_{d \in \{1, 2, \ldots, D\}}\ x_{d} \in [-5, 5]$;
    \item Rastrigin function \cite{rastrigin1974systems}:
    \begin{equation}
        f(\mathbf{x}) = 10D + \sum_{d=1}^{D} \Big{(} x_{d}^2 - 10 \cos \big{(} 2 \pi x_{d} \big{)} \Big{)}
    \end{equation}
    with the box constraints: $\forall_{d \in \{1, 2, \ldots, D\}}\ x_{d} \in [-5, 5]$;
    \item Salomon function \cite{salomon1996re}:
    \begin{equation}
        f(\mathbf{x}) = 1 - \cos \Big{(} 2 \pi \sqrt{\sum_{d=1}^{D}x_d^2} \Big{)} + 0.1 \sqrt{\sum_{d=1}^{D}x_d^2}
    \end{equation}
    with the box constraints: $\forall_{d \in \{1, 2, \ldots, D\}}\ x_{d} \in [-5, 5]$;
    \item Schwefel function \cite{schwefel1981numerical}:
    \begin{equation}
        f(\mathbf{x}) = 418.9829 D - \sum_{d=1}^{D} \Big{(} x_d \sin\big{(} \sqrt{|x_{d}|} \big{)} \Big{)}
    \end{equation}
    with the box constraints: $\forall_{d \in \{1, 2, \ldots, D\}}\ x_{d} \in [200, 500]$.
\end{itemize}

We test the methods on these function with different dimensions, namely, $D \in \{10, 30, 100\}$. We run DEx3, ADE and RevDE for $150$ generations. Since DE evaluates three times less candidate solutions, we run it for $450$ generations to match the number of evaluations. However, we want to highlight that DE is more \textit{informed} than other methods due to the propagation of new solutions in consecutive iterations (\textit{i.e.}, applying the selection mechanism $3$ times more). All experiments are repeated $10$.

\subsubsection{Results \& Discussion} The results of the best solution found until given evaluation for this experiment are presented in Figure \ref{fig:benchmark_results}. First, we notice that ADE and DEx3 similarly to the DE. However, in 3 out of 12 cases (\textit{i.e.}, Griewank with $D=100$), Rastrigin with $D=100$) and Schwefel with $D=100$) DE is able to converge to a better solution than ADE and DEx3. Nevertheless, the results are similar and in the next experiments we skip comparing to DE, because it is not completely fair due to the difference in the number of generations.

Interestingly, ADE performs almost identically as DEx3. This result seems to confirm that it is unnecessary to sample multiple (\textit{i.e.}, three) triplets, and utilizing a single triplet to generate new candidates is sufficient.

In all test cases, RevDE achieved the best results in terms of both final objective value and convergence speed. This result is remarkable, because new candidate solutions are generated \textit{on-the-fly} and are used to generate to new points. Moreover, for $D=30$ and $D=100$, \textit{i.e.}, the higher-dimensional cases, RevDE outperformed other methods significantly. These results are especially promising for real-life applications like parameter values discovery of mechanistic models and computer programs \cite{cranmer2019frontier} or learning controllers in (evolutionary) robotics \cite{lan2018directed, lan2020learning}.

In the Rastrigin function with $D=30$ and $D=100$ there is a peculiar behavior of RevDE where around the evaluation number $80000$ and $50000$, respectively, there is a large improvement in terms of the objective value. We hypothesize that the optimizer ``jumps out'' from a local minimum due to large eigenvalue as discussed in Section \ref{sec:algebraic_properties}.

\begin{figure*}[!htp]
    \centering
    \includegraphics[width=1\textwidth]{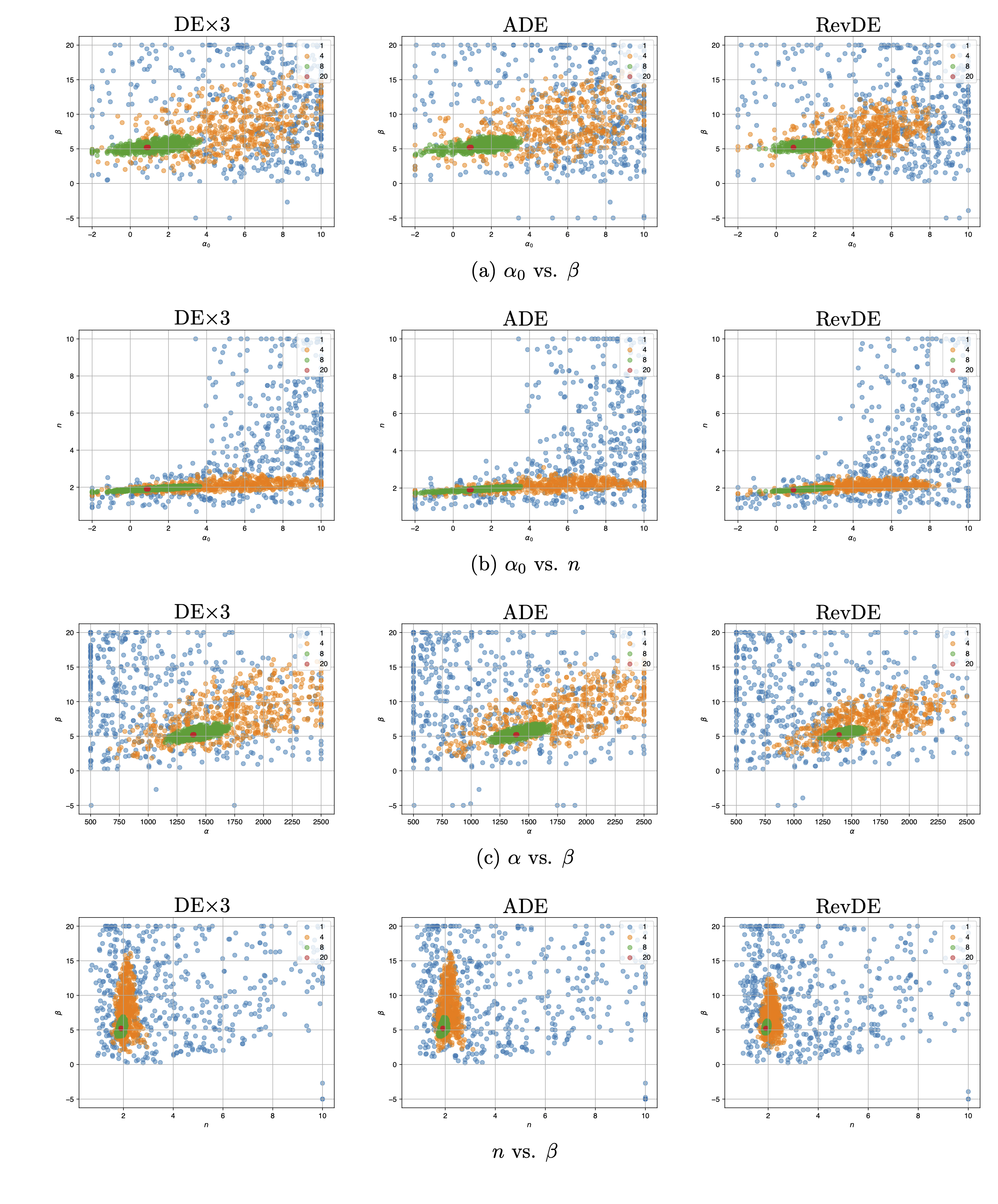}
    \vskip -6mm
    \caption{The discovered parameter values in the repressilator model by: (\textit{left column}) DEx3, (\textit{middle column}) ADE, (\textit{right column}) RevDE. Colors of dots represent the generation number: blue is the 1st generation, orange is the 4th generation, green is the 8th generation, red is the 20th generation. The real parameter values: $(\alpha_0, n, \beta, \alpha) = (1, 2, 5, 1000)$.}
    \label{fig:repressilator_params}
\end{figure*}

\subsection{Gene Repressilator System}

\subsubsection{Details} The gene repressilator system proposed in \cite{elowitz2000synthetic} is a popular model for gene regulatory systems and consists of three genes connected in a feedback loop, where each gene transcribes the repressor protein for the next gene in the loop. The model consists of six ordinary
differential equations that describe dependencies among mRNA $(m_1, m_2, m_3)$ and corresponding proteins $(p_1, p_2, p_3)$, and four parameters $\mathbf{x} = [\alpha_0, n, \beta, \alpha]^{\top}$, which are
as follows:
\begin{align}
    \frac{\mathrm{d} m_1}{\mathrm{d}t} &= -m_1 + \frac{\alpha}{1 + p_3^n} + \alpha_0 \label{eq:repressilator_1}\\
    \frac{\mathrm{d} p_1}{\mathrm{d}t} &= -\beta(p_1 - m_1) \label{eq:repressilator_2}\\
    \frac{\mathrm{d} m_2}{\mathrm{d}t} &= -m_2 + \frac{\alpha}{1 + p_1^n} + \alpha_0 \label{eq:repressilator_3}\\
    \frac{\mathrm{d} p_2}{\mathrm{d}t} &= -\beta(p_2 - m_2) \label{eq:repressilator_4}\\
    \frac{\mathrm{d} m_3}{\mathrm{d}t} &= -m_3 + \frac{\alpha}{1 + p_2^n} + \alpha_0 \label{eq:repressilator_5}\\
    \frac{\mathrm{d} p_3}{\mathrm{d}t} &= -\beta(p_3 - m_3) . \label{eq:repressilator_6}
\end{align}

Further, we assume that only mRNA measurement are measured, while proteins are considered as missing data. The \textbf{goal} of this experiment is to discover the parameters' values for a given observation of mRNA. We transform this problem into the minimization of the following objective:
\begin{equation}\label{eq:repressilator_objective}
    f(\mathbf{x}) = \frac{1}{N} \sum_{n=1}^{N}\sqrt{ \sum_{i=1}^{3} \big{(} m_{i,n} - m_{i,n}(\mathbf{x}) \big{)}^{2}} ,
\end{equation}
where $m_{i,n}(\mathbf{x})$ is given by numerically integrating the system of differential equations in (\ref{eq:repressilator_1}--\ref{eq:repressilator_6}) using a solver, \textit{e.g.}, a Runge-Kutta method\footnote{In this work, we used the explicit Runge-Kutta method of order 5(4) provided by SciPy: \url{https://www.scipy.org/}.}. Notice that the objective function is black-box due to the non-differentiable simulator.

We follow the settings outlined in \cite{toni2009approximate}. The real parameters' values are assumed to be $\mathbf{x}^{*} = [1, 2, 5, 1000]^{\top}$ and we generate real values of $m_i$ by first solving the equations (\ref{eq:repressilator_1}--\ref{eq:repressilator_6}) with $\mathbf{x}^{*}$ and given initial conditions $(m_1, p_1, m_2, p_2, m_3, p_3) = (0, 2, 0, 1, 0, 3)$, and then adding Gaussian noise with the mean equal $0$ and the standard deviation equal $5$.

We run all methods for 20 generations. All experiments were repeated $10$ times. For analyzing final solutions, we look into the convergence of a population from a single run.

\begin{figure}[!htp]
    \centering
    \includegraphics[width=1\columnwidth]{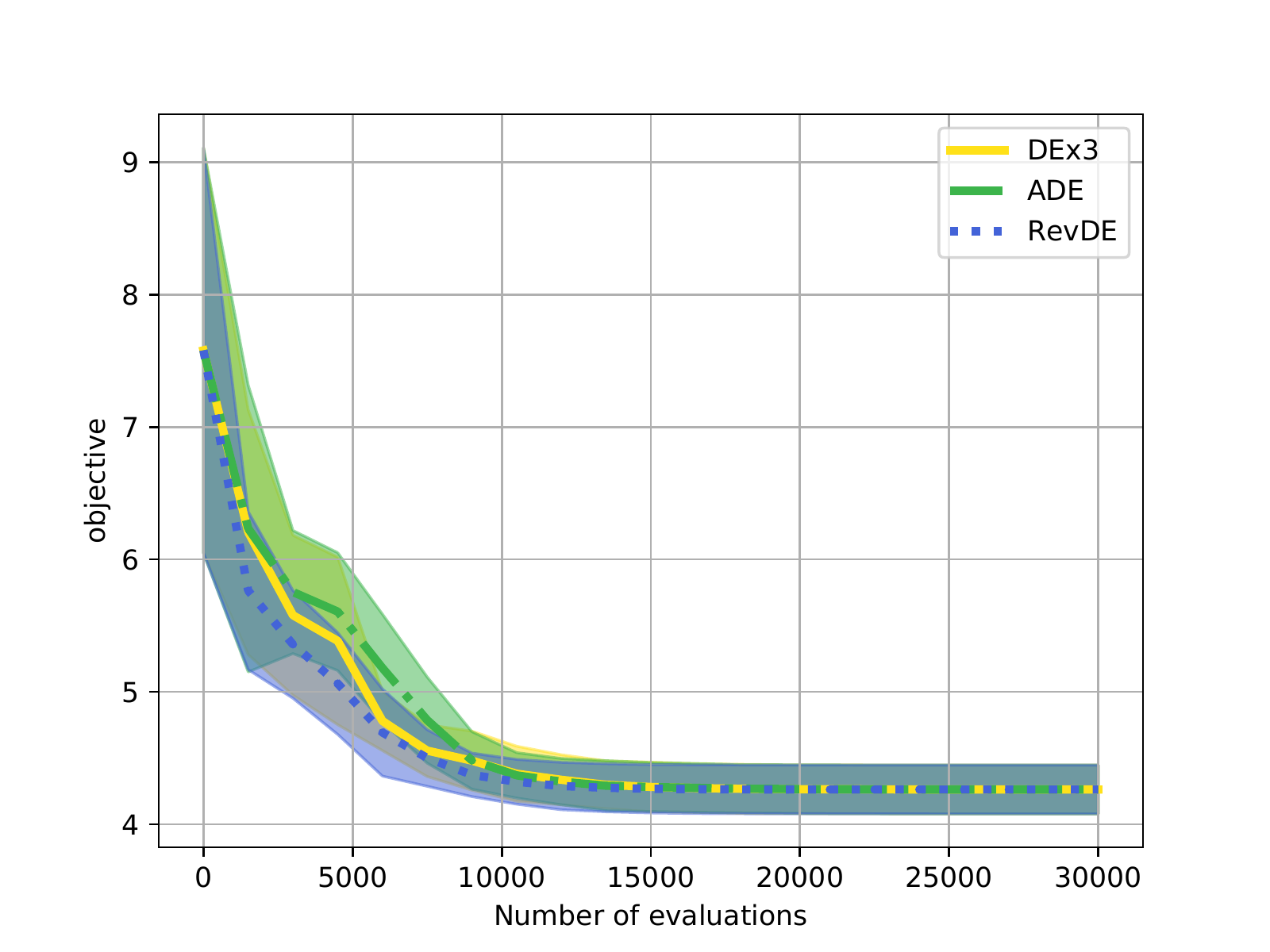}
    \vskip -3mm
    \caption{The results of the best solution found until given fitness evaluation on the repressilator model. The average and the standard deviation over $10$ runs are reported.}
    \label{fig:repressilator_results}
\end{figure}

\subsubsection{Results \& Discussion} We present results of the best solution found until given objective evaluation in Figure \ref{fig:repressilator_results}. Further, we depict converging process of the population in Figure \ref{fig:repressilator_params}. We present only a single run out of ten, however, the behavior is almost indistinguishable between runs.

All methods achieve almost identical objective values (see Figure \ref{fig:repressilator_results}). RevDE seems to converge slightly faster (on average) than other methods, but the difference is not significant.

We can obtain more insight into the performance of the methods by analysing behavior of the population over generations (see Figure \ref{fig:repressilator_params}). First we notice that all methods converge to a single point within $20$ generations. Second, it seems that ADE and DEx3 behave almost indistinguishably. Comparing their scatterplots it is almost impossible to spot a difference. However, RevDE converges faster, because already in the $4$th generation the solutions are less scattered than in the case of DEx3 and ADE. In other words, comparing how points are distributed in the $4$th and the $8$th generation, it is apparent that the variance for the populations found by RevDE is smaller than for ADE ans DEx3. 

Eventually, we would like to comment on discovering the parameter values. In the case of ADE, DEx3 and RevDE all candidates converge to the same solution that are roughly around the point $\mathbf{x} = [1, 2, 5, 1380]^{\top}$. By comparing the real values and the discovered ones we see that the only mismatch is for $\alpha$. However, the discrepancy ($1000$ vs. $1380$) possibly follows from the fact that the observed data is noisy, because the population in the $4$th epoch covered values around $1000$ well. Similarly, in \cite{toni2009approximate} the Approximate Bayesian Computation with the Sequential Monte Carlo method (ABC-SMC) also obtained values between around $800$ and $1300$ (see Figure 4(c) in \cite{toni2009approximate}). We conclude that RevDE seems to be a very promising alternative to Monte Carlo techniques for finding parameters in simulator-based inference problems \cite{cranmer2019frontier}.

\subsection{Neural Networks Learning}

\subsubsection{Details} In the last experiment we aim at evaluating our approach on a high-dimensional optimization problem. For this purpose, we train a neural network with a single fully-connected hidden layer on the image dataset of ten handwritten digits (MNIST \cite{lecun1998gradient}). We resize original images from $28\text{px} \times 28\text{px}$ to $14\text{px} \times 14\text{px}$, and use $20$ hidden units. As a result, we obtain the total number of weights equal $4120$ (\textit{i.e.}, $\mathbb{X} = \mathbb{R}^{4120}$). We use the ReLU non-linear activation function for hidden units and the softmax activation function for outputs. The objective function is the classification error:
\begin{equation}\label{eq:classification_error}
    f(\mathbf{x}) = 1 - \frac{1}{N} \sum_{n=1}^{N} \mathbb{I}[y_{n} = y_{n}(\mathbf{x})] ,
\end{equation}
where $N$ denotes the number of images, $\mathbb{I}[\cdot]$ is an indicator function, $y_n$ is the true label of the $n$th image, and $y_{n}(\mathbf{x})$ is the label for the $n$t image predicted by a neural network with weights $\mathbf{x}$. The prediction of the neural network is a class label with highest value given by the softmax output.

The original dataset consists of $60000$ pairs of images and labels for training, and $10000$ pairs of images and labels for training. In our experiments, we use only $2000$ training points, but all $10000$ testing points. All models are trained for $500$ epochs (generations) and the experiments are repeated $3$ times. For testing, we take a candidate solution from the final population with the lowest training classification error.

The objective function in Eq. \ref{eq:classification_error} is non-differentiable, and, thus, could be treated as a black-box objective. However, we want to highlight that this experiment does not aim at proposing DE as an alternative training procedure to a gradient-based methods, because the log-likelihood function is a good proxy to the objective in (\ref{eq:classification_error}). In fact, is has been shown in multiple papers that the (stochastic) gradient descent optimizer is extremely effective in learning neural networks and DE is not competitive with it at all \cite{ilonen2003differential}. We rather use the neural network learning problem as an interesting showcase of a high-dimensional optimization problem. 

\begin{table}[!htp]
    \centering
    \caption{Test results on MNIST. The average with the standard errors over $3$ runs are reported.}
    \vskip -0mm
    \begin{tabular}{c|c}
        \textbf{Method} & \textbf{Classification error} \\
        \hline
        DE$\times 3$    & 20.1 $\pm$ 1.4 \\
        ADE             & 18.1 $\pm$ 0.2 \\
        RevDE           & 18.5 $\pm$ 0.8 
    \end{tabular}
    \label{tab:mnist_test}
\end{table}

\subsubsection{Results \& Discussion} In Figure \ref{fig:mnist_training} we present learning curves for neural networks trained with different methods. Additionally, in Table \ref{tab:mnist_test} we gather test classification errors.

First of all, we notice that the training is not fully converged and possibly better results could have been achieved. Nevertheless, our goal is to present performance of our methods on a high-dimensional problem rather than reaching state-of-the art scores. That being said, we first observe that the proposed extensions of DE shared similar learning curves. ADE performed the best during training, and RevDE converged to a better point than DEx3 in the very end. However, the final test performance was better for ADE and RevDE than DEx3. This result could be possibly explained by the fact that DEx3 is more stochastic than the other two methods that could be harmful in the highly dimensional problem.

A close inspection of the results in Table \ref{tab:mnist_test} suggests that ADE and RevDE perform on par, and they seem to be better than DEx3. This result is potentially interesting because the negative message delivered in \cite{ilonen2003differential} that DE is definitely worse than the gradient-based learning method is not necessarily true and more research in this direction is required. Especially in the context of adding non-differentiable components (regularizers) to the learning objective.

\begin{figure}[!htp]
    \centering
    \includegraphics[width=1\columnwidth]{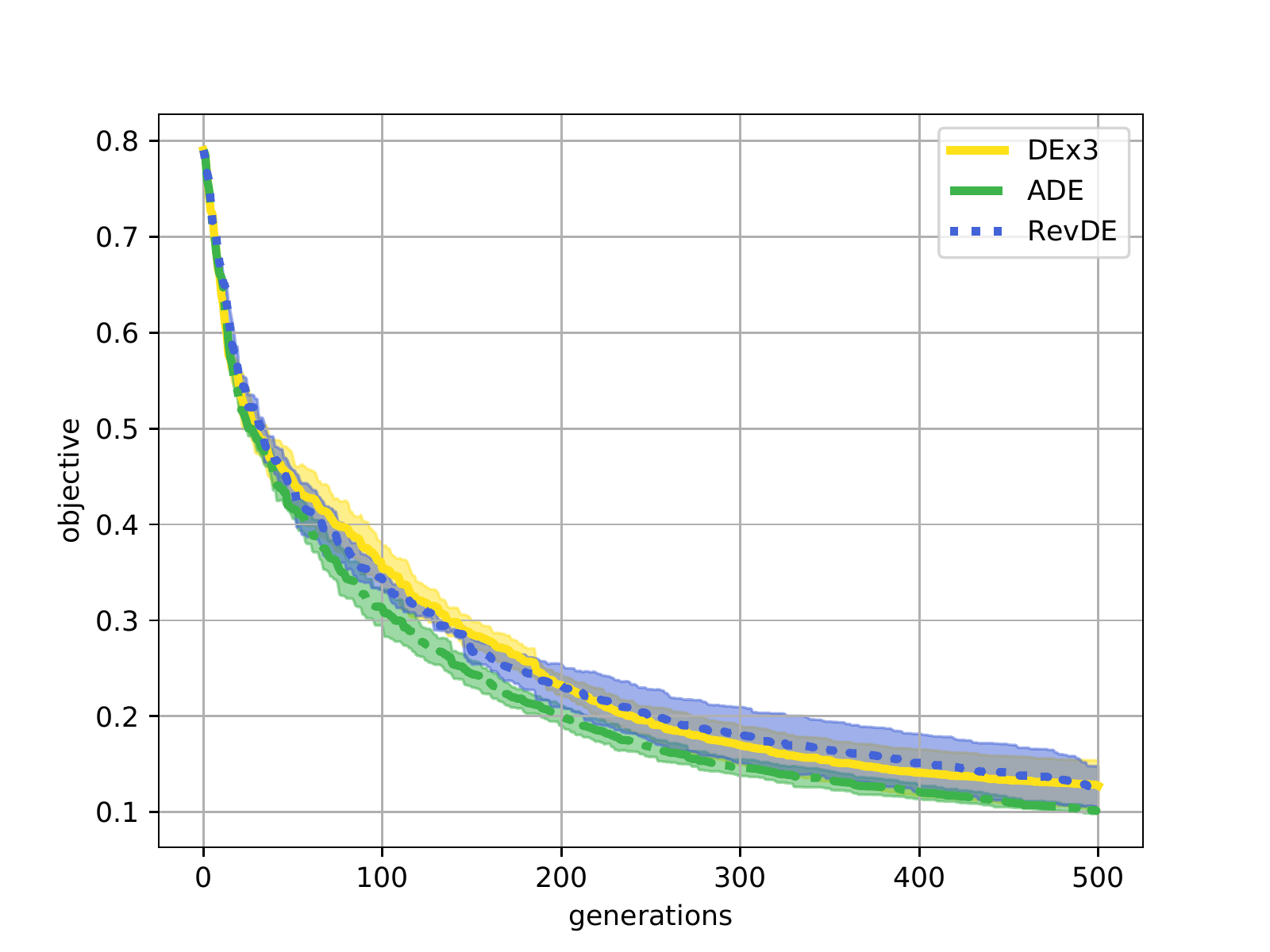}
    \vskip -2mm
    \caption{Training curves on MNIST. The average and the standard deviation over $3$ runs are reported.}
    \label{fig:mnist_training}
\end{figure}

%% file: conclusion.tex
In this paper, we note that insufficient variability of the population could cause DE to loose diversity too quickly. 
In order to counteract this issue, we propose three extensions of DE: (i) DE with multiple samples of candidates for calculating perturbations, (ii) DE with the reversible linear transformation using a sum of the identity matrix and an anti-symmetric matrix, (iii) DE with the reversible linear transformation utilizing newly generated yet not evaluated candidates. 

We provide a theoretical analysis of the proposed linear operators by proving their reversibility, and inspecting their eigenvalues.
Further, we show empirically on three testbeds (benchmark function optimization, discovering parameter values of the gene repressilator systems, and learning neural networks) that producing new candidates \textit{on-the-fly} allows to obtain better results in fewer number of evaluations compared to DE.

We believe that this work opens new possible research directions:
\begin{itemize}
    \item Representing the differential mutation as a linear transformation allows to look into other forms of linear operators.
    \item The linear operators defined in this paper are parameterized with a single parameter. A natural extension would be considering different parameterization.
    \item Here, we present an analysis based on eigenvalues. However, we can consider the reversible transformation as a dynamical system (\textit{e.g.}, an extension of the analysis outlined in \cite{opara2019differential}).
    \item We can take advantage of the reversibility of the proposed linear transformations. For instance, reversibility is an important property of transition operators in MCMC methods \cite{andrieu2003introduction}. A modification of the vanilla DE for formulating a proper MCMC method was already presented in \cite{ter2006markov} and an interesting direction would be to extend this work using DE with the reversible linear transformations.
\end{itemize}

%% file: appendix.tex
\subsection*{Non-singularity of $\mathbf{M}$ and $\mathbf{R}$}

\begin{proposition}\label{prop:ade}
The matrix $\mathbf{M}$ defined in Eq. \ref{eq:eye_anti} is non-singular.
\end{proposition}
\begin{proof}
Since the matrix $\mathbf{M}$ is a small 3-on-3 matrix, we can calculate its determinant analytically that gives:
\begin{equation}\label{eq:m_det}
    \mathrm{det}(\mathbf{M}) = 1 + 3 F^2 .
\end{equation}
For any value of $F$ we have $\mathrm{det}(\mathbf{M}) \neq 0$, therefore, the matrix $\mathbf{M}$ is non-singular
\end{proof}

\begin{proposition}\label{prop:revde}
The matrix $\mathbf{R}$ defined in Eq. \ref{eq:r} is non-singular.
\end{proposition}
\begin{proof}
The matrix $\mathbf{R}$ is a small 3-on-3 matrix, thus, we can calculate its determinant analytically, that gives:
\begin{equation}\label{eq:r_det}
    \mathrm{det}(\mathbf{R}) = 1.
\end{equation}
Since the determinant is always $1$, then $\mathbf{R}$ is non-singular.
\end{proof}

%% file: main.bbl

\begin{thebibliography}{35}


\ifx \showCODEN    \undefined \def \showCODEN     #1{\unskip}     \fi
\ifx \showDOI      \undefined \def \showDOI       #1{#1}\fi
\ifx \showISBNx    \undefined \def \showISBNx     #1{\unskip}     \fi
\ifx \showISBNxiii \undefined \def \showISBNxiii  #1{\unskip}     \fi
\ifx \showISSN     \undefined \def \showISSN      #1{\unskip}     \fi
\ifx \showLCCN     \undefined \def \showLCCN      #1{\unskip}     \fi
\ifx \shownote     \undefined \def \shownote      #1{#1}          \fi
\ifx \showarticletitle \undefined \def \showarticletitle #1{#1}   \fi
\ifx \showURL      \undefined \def \showURL       {\relax}        \fi
\providecommand\bibfield[2]{#2}
\providecommand\bibinfo[2]{#2}
\providecommand\natexlab[1]{#1}
\providecommand\showeprint[2][]{arXiv:#2}

\bibitem[\protect\citeauthoryear{Andrieu, De~Freitas, Doucet, and
  Jordan}{Andrieu et~al\mbox{.}}{2003}]%
        {andrieu2003introduction}
\bibfield{author}{\bibinfo{person}{Christophe Andrieu}, \bibinfo{person}{Nando
  De~Freitas}, \bibinfo{person}{Arnaud Doucet}, {and}
  \bibinfo{person}{Michael~I. Jordan}.} \bibinfo{year}{2003}\natexlab{}.
\newblock \showarticletitle{An introduction to MCMC for machine learning}.
\newblock \bibinfo{journal}{{\em Machine Learning\/}} \bibinfo{volume}{50},
  \bibinfo{number}{1-2} (\bibinfo{year}{2003}), \bibinfo{pages}{5--43}.
\newblock


\bibitem[\protect\citeauthoryear{Audet and Hare}{Audet and Hare}{2017}]%
        {audet2017derivative}
\bibfield{author}{\bibinfo{person}{Charles Audet} {and} \bibinfo{person}{Warren
  Hare}.} \bibinfo{year}{2017}\natexlab{}.
\newblock \bibinfo{booktitle}{{\em {Derivative-free and Blackbox
  Optimization}}}.
\newblock \bibinfo{publisher}{Springer}.
\newblock


\bibitem[\protect\citeauthoryear{B\"ack}{B\"ack}{1996}]%
        {back1996evolutionary}
\bibfield{author}{\bibinfo{person}{Thomas B\"ack}.}
  \bibinfo{year}{1996}\natexlab{}.
\newblock \bibinfo{booktitle}{{\em {Evolutionary Algorithms in Theory and
  Practice: Evolution Strategies, Evolutionary Programming, Genetic
  Algorithms}}}.
\newblock \bibinfo{publisher}{Oxford university press}.
\newblock


\bibitem[\protect\citeauthoryear{B{\"a}ck, Foussette, and Krause}{B{\"a}ck
  et~al\mbox{.}}{2013}]%
        {back2013contemporary}
\bibfield{author}{\bibinfo{person}{Thomas B{\"a}ck},
  \bibinfo{person}{Christophe Foussette}, {and} \bibinfo{person}{Peter
  Krause}.} \bibinfo{year}{2013}\natexlab{}.
\newblock \bibinfo{booktitle}{{\em {Contemporary Evolution Strategies}}}.
\newblock \bibinfo{publisher}{Springer}.
\newblock


\bibitem[\protect\citeauthoryear{Bubnicki}{Bubnicki}{2005}]%
        {bubnicki2005modern}
\bibfield{author}{\bibinfo{person}{Zdzis{\l}aw Bubnicki}.}
  \bibinfo{year}{2005}\natexlab{}.
\newblock \bibinfo{booktitle}{{\em {Modern Control Theory}}}.
\newblock \bibinfo{publisher}{Springer}.
\newblock


\bibitem[\protect\citeauthoryear{Cranmer, Brehmer, and Louppe}{Cranmer
  et~al\mbox{.}}{2019}]%
        {cranmer2019frontier}
\bibfield{author}{\bibinfo{person}{Kyle Cranmer}, \bibinfo{person}{Johann
  Brehmer}, {and} \bibinfo{person}{Gilles Louppe}.}
  \bibinfo{year}{2019}\natexlab{}.
\newblock \showarticletitle{The frontier of simulation-based inference}.
\newblock \bibinfo{journal}{{\em arXiv preprint arXiv:1911.01429\/}}
  (\bibinfo{year}{2019}).
\newblock


\bibitem[\protect\citeauthoryear{Doncieux, Bredeche, Mouret, and
  Eiben}{Doncieux et~al\mbox{.}}{2015}]%
        {doncieux2015evolutionary}
\bibfield{author}{\bibinfo{person}{Stephane Doncieux}, \bibinfo{person}{Nicolas
  Bredeche}, \bibinfo{person}{Jean-Baptiste Mouret}, {and}
  \bibinfo{person}{Agoston~E. Eiben}.} \bibinfo{year}{2015}\natexlab{}.
\newblock \showarticletitle{Evolutionary robotics: what, why, and where to}.
\newblock \bibinfo{journal}{{\em Frontiers in Robotics and AI\/}}
  \bibinfo{volume}{2} (\bibinfo{year}{2015}), \bibinfo{pages}{4}.
\newblock


\bibitem[\protect\citeauthoryear{Eiben, Marchiori, and Valko}{Eiben
  et~al\mbox{.}}{2004}]%
        {eiben2004evolutionary}
\bibfield{author}{\bibinfo{person}{Agoston~E. Eiben}, \bibinfo{person}{Elena
  Marchiori}, {and} \bibinfo{person}{VA Valko}.}
  \bibinfo{year}{2004}\natexlab{}.
\newblock \showarticletitle{Evolutionary algorithms with on-the-fly population
  size adjustment}. In \bibinfo{booktitle}{{\em International Conference on
  Parallel Problem Solving from Nature}}. Springer, \bibinfo{pages}{41--50}.
\newblock


\bibitem[\protect\citeauthoryear{Eiben and Smith}{Eiben and Smith}{2015}]%
        {eiben2003introduction}
\bibfield{author}{\bibinfo{person}{Agoston~E. Eiben} {and}
  \bibinfo{person}{James~E. Smith}.} \bibinfo{year}{2015}\natexlab{}.
\newblock \bibinfo{booktitle}{{\em {Introduction to Evolutionary Computing}}}.
  Vol.~\bibinfo{volume}{53}.
\newblock \bibinfo{publisher}{Springer}.
\newblock


\bibitem[\protect\citeauthoryear{Elowitz and Leibler}{Elowitz and
  Leibler}{2000}]%
        {elowitz2000synthetic}
\bibfield{author}{\bibinfo{person}{Michael~B Elowitz} {and}
  \bibinfo{person}{Stanislas Leibler}.} \bibinfo{year}{2000}\natexlab{}.
\newblock \showarticletitle{A synthetic oscillatory network of transcriptional
  regulators}.
\newblock \bibinfo{journal}{{\em Nature\/}} \bibinfo{volume}{403},
  \bibinfo{number}{6767} (\bibinfo{year}{2000}), \bibinfo{pages}{335--338}.
\newblock


\bibitem[\protect\citeauthoryear{Griewank}{Griewank}{1981}]%
        {griewank1981generalized}
\bibfield{author}{\bibinfo{person}{Andreas~O. Griewank}.}
  \bibinfo{year}{1981}\natexlab{}.
\newblock \showarticletitle{Generalized descent for global optimization}.
\newblock \bibinfo{journal}{{\em Journal of optimization theory and
  applications\/}} \bibinfo{volume}{34}, \bibinfo{number}{1}
  (\bibinfo{year}{1981}), \bibinfo{pages}{11--39}.
\newblock


\bibitem[\protect\citeauthoryear{Ilonen, Kamarainen, and Lampinen}{Ilonen
  et~al\mbox{.}}{2003}]%
        {ilonen2003differential}
\bibfield{author}{\bibinfo{person}{Jarmo Ilonen},
  \bibinfo{person}{Joni-Kristian Kamarainen}, {and} \bibinfo{person}{Jouni
  Lampinen}.} \bibinfo{year}{2003}\natexlab{}.
\newblock \showarticletitle{Differential evolution training algorithm for
  feed-forward neural networks}.
\newblock \bibinfo{journal}{{\em Neural Processing Letters\/}}
  \bibinfo{volume}{17}, \bibinfo{number}{1} (\bibinfo{year}{2003}),
  \bibinfo{pages}{93--105}.
\newblock


\bibitem[\protect\citeauthoryear{Jones, Schonlau, and Welch}{Jones
  et~al\mbox{.}}{1998}]%
        {jones1998efficient}
\bibfield{author}{\bibinfo{person}{Donald~R. Jones}, \bibinfo{person}{Matthias
  Schonlau}, {and} \bibinfo{person}{William~J. Welch}.}
  \bibinfo{year}{1998}\natexlab{}.
\newblock \showarticletitle{Efficient global optimization of expensive
  black-box functions}.
\newblock \bibinfo{journal}{{\em Journal of Global optimization\/}}
  \bibinfo{volume}{13}, \bibinfo{number}{4} (\bibinfo{year}{1998}),
  \bibinfo{pages}{455--492}.
\newblock


\bibitem[\protect\citeauthoryear{Karaboga}{Karaboga}{2005}]%
        {karaboga2005digital}
\bibfield{author}{\bibinfo{person}{Nurhan Karaboga}.}
  \bibinfo{year}{2005}\natexlab{}.
\newblock \showarticletitle{{Digital IIR filter design using differential
  evolution algorithm}}.
\newblock \bibinfo{journal}{{\em EURASIP Journal on Advances in Signal
  Processing\/}} \bibinfo{volume}{2005}, \bibinfo{number}{8}
  (\bibinfo{year}{2005}), \bibinfo{pages}{856824}.
\newblock


\bibitem[\protect\citeauthoryear{Lan, De~Carlo, van Diggelen, Tomczak, Roijers,
  and Eiben}{Lan et~al\mbox{.}}{2020}]%
        {lan2020learning}
\bibfield{author}{\bibinfo{person}{Gongjin Lan}, \bibinfo{person}{Matteo
  De~Carlo}, \bibinfo{person}{Fuda van Diggelen}, \bibinfo{person}{Jakub~M.
  Tomczak}, \bibinfo{person}{Diederik~M. Roijers}, {and}
  \bibinfo{person}{Agoston~E. Eiben}.} \bibinfo{year}{2020}\natexlab{}.
\newblock \showarticletitle{Learning Directed Locomotion in Modular Robots with
  Evolvable Morphologies}.
\newblock \bibinfo{journal}{{\em arXiv preprint arXiv:2001.07804\/}}
  (\bibinfo{year}{2020}).
\newblock


\bibitem[\protect\citeauthoryear{Lan, Jelisavcic, Roijers, Haasdijk, and
  Eiben}{Lan et~al\mbox{.}}{2018}]%
        {lan2018directed}
\bibfield{author}{\bibinfo{person}{Gongjin Lan}, \bibinfo{person}{Milan
  Jelisavcic}, \bibinfo{person}{Diederik~M. Roijers}, \bibinfo{person}{Evert
  Haasdijk}, {and} \bibinfo{person}{Agoston~E. Eiben}.}
  \bibinfo{year}{2018}\natexlab{}.
\newblock \showarticletitle{Directed locomotion for modular robots with
  evolvable morphologies}. In \bibinfo{booktitle}{{\em International Conference
  on Parallel Problem Solving from Nature}}. Springer,
  \bibinfo{pages}{476--487}.
\newblock


\bibitem[\protect\citeauthoryear{LeCun, Bottou, Bengio, and Haffner}{LeCun
  et~al\mbox{.}}{1998}]%
        {lecun1998gradient}
\bibfield{author}{\bibinfo{person}{Yann LeCun}, \bibinfo{person}{L{\'e}on
  Bottou}, \bibinfo{person}{Yoshua Bengio}, {and} \bibinfo{person}{Patrick
  Haffner}.} \bibinfo{year}{1998}\natexlab{}.
\newblock \showarticletitle{Gradient-based learning applied to document
  recognition}.
\newblock \bibinfo{journal}{{\it Proc. IEEE}} \bibinfo{volume}{86},
  \bibinfo{number}{11} (\bibinfo{year}{1998}), \bibinfo{pages}{2278--2324}.
\newblock


\bibitem[\protect\citeauthoryear{Martinez-Soltero and
  Hernandez-Barragan}{Martinez-Soltero and Hernandez-Barragan}{2018}]%
        {martinez2018robot}
\bibfield{author}{\bibinfo{person}{Erasmo~Gabriel Martinez-Soltero} {and}
  \bibinfo{person}{Jesus Hernandez-Barragan}.} \bibinfo{year}{2018}\natexlab{}.
\newblock \showarticletitle{Robot navigation based on differential evolution}.
\newblock \bibinfo{journal}{{\em IFAC-PapersOnLine\/}} \bibinfo{volume}{51},
  \bibinfo{number}{13} (\bibinfo{year}{2018}), \bibinfo{pages}{350--354}.
\newblock


\bibitem[\protect\citeauthoryear{Noman and Iba}{Noman and Iba}{2008}]%
        {noman2008accelerating}
\bibfield{author}{\bibinfo{person}{Nasimul Noman} {and}
  \bibinfo{person}{Hitoshi Iba}.} \bibinfo{year}{2008}\natexlab{}.
\newblock \showarticletitle{Accelerating differential evolution using an
  adaptive local search}.
\newblock \bibinfo{journal}{{\em IEEE Transactions on evolutionary
  Computation\/}} \bibinfo{volume}{12}, \bibinfo{number}{1}
  (\bibinfo{year}{2008}), \bibinfo{pages}{107--125}.
\newblock


\bibitem[\protect\citeauthoryear{Opara and Arabas}{Opara and Arabas}{2019}]%
        {opara2019differential}
\bibfield{author}{\bibinfo{person}{Karol~R Opara} {and}
  \bibinfo{person}{Jaros{\l}aw Arabas}.} \bibinfo{year}{2019}\natexlab{}.
\newblock \showarticletitle{{Differential Evolution: A survey of theoretical
  analyses}}.
\newblock \bibinfo{journal}{{\em Swarm and evolutionary computation\/}}
  \bibinfo{volume}{44} (\bibinfo{year}{2019}), \bibinfo{pages}{546--558}.
\newblock


\bibitem[\protect\citeauthoryear{Pedersen}{Pedersen}{2010}]%
        {pedersen2010good}
\bibfield{author}{\bibinfo{person}{Magnus Erik~Hvass Pedersen}.}
  \bibinfo{year}{2010}\natexlab{}.
\newblock \bibinfo{booktitle}{{\em Good parameters for differential
  evolution}}.
\newblock \bibinfo{type}{{T}echnical {R}eport} HL1002.
  \bibinfo{institution}{Hvass Laboratories}.
\newblock


\bibitem[\protect\citeauthoryear{Price, Storn, and Lampinen}{Price
  et~al\mbox{.}}{2006}]%
        {price2006differential}
\bibfield{author}{\bibinfo{person}{Kenneth Price}, \bibinfo{person}{Rainer~M
  Storn}, {and} \bibinfo{person}{Jouni~A. Lampinen}.}
  \bibinfo{year}{2006}\natexlab{}.
\newblock \bibinfo{booktitle}{{\em {Differential Evolution: A Practical
  Approach to Global Optimization}}}.
\newblock \bibinfo{publisher}{Springer Science \& Business Media}.
\newblock


\bibitem[\protect\citeauthoryear{Rastrigin}{Rastrigin}{1974}]%
        {rastrigin1974systems}
\bibfield{author}{\bibinfo{person}{Leonard~Andreevi{\v{c}} Rastrigin}.}
  \bibinfo{year}{1974}\natexlab{}.
\newblock \showarticletitle{Systems of extremal control}.
\newblock \bibinfo{journal}{{\em Nauka\/}} (\bibinfo{year}{1974}).
\newblock


\bibitem[\protect\citeauthoryear{Salomon}{Salomon}{1996}]%
        {salomon1996re}
\bibfield{author}{\bibinfo{person}{Ralf Salomon}.}
  \bibinfo{year}{1996}\natexlab{}.
\newblock \showarticletitle{Re-evaluating genetic algorithm performance under
  coordinate rotation of benchmark functions. A survey of some theoretical and
  practical aspects of genetic algorithms}.
\newblock \bibinfo{journal}{{\em BioSystems\/}} \bibinfo{volume}{39},
  \bibinfo{number}{3} (\bibinfo{year}{1996}), \bibinfo{pages}{263--278}.
\newblock


\bibitem[\protect\citeauthoryear{Schwefel}{Schwefel}{1981}]%
        {schwefel1981numerical}
\bibfield{author}{\bibinfo{person}{Hans-Paul Schwefel}.}
  \bibinfo{year}{1981}\natexlab{}.
\newblock \bibinfo{booktitle}{{\em Numerical optimization of computer models}}.
\newblock \bibinfo{publisher}{John Wiley \& Sons, Inc.}
\newblock


\bibitem[\protect\citeauthoryear{Shahriari, Swersky, Wang, Adams, and
  De~Freitas}{Shahriari et~al\mbox{.}}{2015}]%
        {shahriari2015taking}
\bibfield{author}{\bibinfo{person}{Bobak Shahriari}, \bibinfo{person}{Kevin
  Swersky}, \bibinfo{person}{Ziyu Wang}, \bibinfo{person}{Ryan~P Adams}, {and}
  \bibinfo{person}{Nando De~Freitas}.} \bibinfo{year}{2015}\natexlab{}.
\newblock \showarticletitle{{Taking the human out of the loop: A review of
  Bayesian optimization}}.
\newblock \bibinfo{journal}{{\it Proc. IEEE}} \bibinfo{volume}{104},
  \bibinfo{number}{1} (\bibinfo{year}{2015}), \bibinfo{pages}{148--175}.
\newblock


\bibitem[\protect\citeauthoryear{Sharma, Komninos, L{\'o}pez-Ib{\'a}{\~n}ez,
  and Kazakov}{Sharma et~al\mbox{.}}{2019}]%
        {sharma2019deep}
\bibfield{author}{\bibinfo{person}{Mudita Sharma}, \bibinfo{person}{Alexandros
  Komninos}, \bibinfo{person}{Manuel L{\'o}pez-Ib{\'a}{\~n}ez}, {and}
  \bibinfo{person}{Dimitar Kazakov}.} \bibinfo{year}{2019}\natexlab{}.
\newblock \showarticletitle{Deep reinforcement learning based parameter control
  in differential evolution}. In \bibinfo{booktitle}{{\em Proceedings of the
  Genetic and Evolutionary Computation Conference}}. \bibinfo{pages}{709--717}.
\newblock


\bibitem[\protect\citeauthoryear{Storn and Price}{Storn and Price}{1997}]%
        {storn1997differential}
\bibfield{author}{\bibinfo{person}{Rainer Storn} {and} \bibinfo{person}{Kenneth
  Price}.} \bibinfo{year}{1997}\natexlab{}.
\newblock \showarticletitle{Differential evolution--a simple and efficient
  heuristic for global optimization over continuous spaces}.
\newblock \bibinfo{journal}{{\em Journal of Global Optimization\/}}
  \bibinfo{volume}{11}, \bibinfo{number}{4} (\bibinfo{year}{1997}),
  \bibinfo{pages}{341--359}.
\newblock


\bibitem[\protect\citeauthoryear{Strang}{Strang}{2016}]%
        {strang1993introduction}
\bibfield{author}{\bibinfo{person}{Gilbert Strang}.}
  \bibinfo{year}{2016}\natexlab{}.
\newblock \bibinfo{booktitle}{{\em {Introduction to Linear Algebra}}}.
\newblock \bibinfo{publisher}{Wellesley-Cambridge Press Wellesley, MA}.
\newblock


\bibitem[\protect\citeauthoryear{Tasoulis, Plagianakos, and Vrahatis}{Tasoulis
  et~al\mbox{.}}{2006}]%
        {tasoulis2006differential}
\bibfield{author}{\bibinfo{person}{Dimitris~K. Tasoulis},
  \bibinfo{person}{Vassilis~P. Plagianakos}, {and} \bibinfo{person}{Michael~N.
  Vrahatis}.} \bibinfo{year}{2006}\natexlab{}.
\newblock \showarticletitle{Differential evolution algorithms for finding
  predictive gene subsets in microarray data}. In \bibinfo{booktitle}{{\em IFIP
  International Conference on Artificial Intelligence Applications and
  Innovations}}. Springer, \bibinfo{pages}{484--491}.
\newblock


\bibitem[\protect\citeauthoryear{Ter~Braak}{Ter~Braak}{2006}]%
        {ter2006markov}
\bibfield{author}{\bibinfo{person}{Cajo~J.F. Ter~Braak}.}
  \bibinfo{year}{2006}\natexlab{}.
\newblock \showarticletitle{{A Markov Chain Monte Carlo version of the genetic
  algorithm Differential Evolution: easy Bayesian computing for real parameter
  spaces}}.
\newblock \bibinfo{journal}{{\em Statistics and Computing\/}}
  \bibinfo{volume}{16}, \bibinfo{number}{3} (\bibinfo{year}{2006}),
  \bibinfo{pages}{239--249}.
\newblock


\bibitem[\protect\citeauthoryear{Tomczak and W{\k{e}}glarz-Tomczak}{Tomczak and
  W{\k{e}}glarz-Tomczak}{2019}]%
        {tomczak2019estimating}
\bibfield{author}{\bibinfo{person}{Jakub~M. Tomczak} {and}
  \bibinfo{person}{Ewelina W{\k{e}}glarz-Tomczak}.}
  \bibinfo{year}{2019}\natexlab{}.
\newblock \showarticletitle{{Estimating kinetic constants in the
  Michaelis--Menten model from one enzymatic assay using Approximate Bayesian
  Computation}}.
\newblock \bibinfo{journal}{{\em FEBS letters\/}} \bibinfo{volume}{593},
  \bibinfo{number}{19} (\bibinfo{year}{2019}), \bibinfo{pages}{2742--2750}.
\newblock


\bibitem[\protect\citeauthoryear{Toni, Welch, Strelkowa, Ipsen, and
  Stumpf}{Toni et~al\mbox{.}}{2009}]%
        {toni2009approximate}
\bibfield{author}{\bibinfo{person}{Tina Toni}, \bibinfo{person}{David Welch},
  \bibinfo{person}{Natalja Strelkowa}, \bibinfo{person}{Andreas Ipsen}, {and}
  \bibinfo{person}{Michael~P.H. Stumpf}.} \bibinfo{year}{2009}\natexlab{}.
\newblock \showarticletitle{{Approximate Bayesian computation scheme for
  parameter inference and model selection in dynamical systems}}.
\newblock \bibinfo{journal}{{\em Journal of the Royal Society Interface\/}}
  \bibinfo{volume}{6}, \bibinfo{number}{31} (\bibinfo{year}{2009}),
  \bibinfo{pages}{187--202}.
\newblock


\bibitem[\protect\citeauthoryear{Wang, Su, and Jang}{Wang
  et~al\mbox{.}}{2001}]%
        {wang2001kinetic}
\bibfield{author}{\bibinfo{person}{Feng-Sheng Wang}, \bibinfo{person}{Tzu-Liang
  Su}, {and} \bibinfo{person}{Horng-Jhy Jang}.}
  \bibinfo{year}{2001}\natexlab{}.
\newblock \showarticletitle{{Hybrid Differential Evolution for Problems of
  Kinetic Parameter Estimation and Dynamic Optimization of an Ethanol
  Fermentation Process}}.
\newblock \bibinfo{journal}{{\em Industrial \& Engineering Chemistry
  Research\/}} \bibinfo{volume}{40}, \bibinfo{number}{13}
  (\bibinfo{year}{2001}), \bibinfo{pages}{2876--2885}.
\newblock


\bibitem[\protect\citeauthoryear{Zhang and Sanderson}{Zhang and
  Sanderson}{2009}]%
        {zhang2009jade}
\bibfield{author}{\bibinfo{person}{Jingqiao Zhang} {and}
  \bibinfo{person}{Arthur~C. Sanderson}.} \bibinfo{year}{2009}\natexlab{}.
\newblock \showarticletitle{{JADE: Adaptive Differential Evolution with
  Optional External Archive}}.
\newblock \bibinfo{journal}{{\em IEEE Transactions on evolutionary
  computation\/}} \bibinfo{volume}{13}, \bibinfo{number}{5}
  (\bibinfo{year}{2009}), \bibinfo{pages}{945--958}.
\newblock


\end{thebibliography}
